\documentclass{article}
\usepackage{arxiv,times}
\usepackage{geometry} 
\usepackage{caption}

\usepackage{amsmath, amsthm, amssymb, bm}
\newtheorem{Definition}{Definition}
\newtheorem{Theorem}{Theorem}
\newtheorem*{Theorem*}{Theorem}
\newtheorem{Corollary}{Corollary}

\usepackage[colorlinks=true, linkcolor=blue, urlcolor=cyan]{hyperref}

\usepackage{multirow}

\usepackage{booktabs}
\usepackage{longtable}
\usepackage{array}

\usepackage{graphicx}
\usepackage{subfig, caption, wrapfig}
\usepackage{microtype}   
\usepackage[utf8]{inputenc} 
\usepackage[T1]{fontenc}  
\usepackage{url}      
\usepackage{nicefrac}    
\usepackage{scalefnt}
\usepackage{natbib}
\usepackage{xcolor}

\title{Beyond Cosine Similarity}

\author{
Xinbo Ai \\
School of Intelligent Engineering and Automation\\
Beijing University of Posts and Telecommunications\\
Beijing 100876, China\\
\texttt{axb@bupt.edu.cn}
}

\date{} 
\begin{document}

\maketitle

\vspace{3mm}

\begin{abstract}
  Cosine similarity, the standard metric for measuring semantic similarity in vector spaces, is mathematically grounded in the Cauchy-Schwarz inequality, which inherently limits it to capturing linear relationships—a constraint that fails to model the complex, nonlinear structures of real-world semantic spaces. We advance this theoretical underpinning by deriving a tighter upper bound for the dot product than the classical Cauchy-Schwarz bound. This new bound leads directly to \textit{recos}, a similarity metric that normalizes the dot product by the sorted vector components. \textit{recos} relaxes the condition for perfect similarity from strict linear dependence to ordinal concordance, thereby capturing a broader class of relationships. Extensive experiments across 11 embedding models—spanning static, contextualized, and universal types—demonstrate that \textit{recos} consistently outperforms traditional cosine similarity, achieving higher correlation with human judgments on standard Semantic Textual Similarity (STS) benchmarks. Our work establishes \textit{recos} as a mathematically principled and empirically superior alternative, offering enhanced accuracy for semantic analysis in complex embedding spaces.
\end{abstract}

\section{Introduction}

Since its introduction for the vector space model of automatic indexing \cite{saltonVectorSpaceModel1975}, cosine similarity has become the \textit{de facto} standard for measuring text similarity. Its scale invariance and computational efficiency make it ubiquitous in applications ranging from document clustering \cite{saltonVectorSpaceModel1975} and static word embedding comparisons \cite{mikolovEfficientEstimationWord2013,pennington2014glove} to semantic analysis with modern large language models \cite{brownLanguageModelsAre2020}. Its foundational role persists in contemporary applications including retrieval-augmented generation \cite{lewisRetrievalAugmentedGenerationKnowledgeIntensive2020} and cross-modal alignment \cite{radfordLearningTransferableVisual2021}.

The mathematical basis for cosine similarity is the Cauchy-Schwarz inequality, a cornerstone of linear algebra which establishes an upper bound for the dot product of two vectors. Cosine similarity leverages this bound by normalizing the dot product, yielding a measure that depends solely on the angle between vectors and is invariant to their magnitudes. This property is particularly advantageous in high-dimensional embedding spaces.

Despite its dominance, recent work questions the reliability of cosine similarity for capturing semantic relations \cite{steckCosineSimilarityEmbeddingsReally2024}. A key criticism concerns frequency bias: \cite{zhouProblemsCosineMeasure2022} demonstrate that cosine similarity systematically underestimates similarity for high-frequency words in contextualized embeddings, attributing this to the inflated $\ell_2$ norms of such words. Proposed mitigations often involve post-processing or integrating norm-based components \cite{liDistanceWeightedCosine2013, xiaLearningSimilarityCosine2015, wannasuphoprasitSolvingCosineSimilarity2023}, or modifying the feature space \cite{sidorovSoftSimilaritySoft2014}. While effective, these approaches typically \textit{adapt} the application of cosine similarity rather than \textit{reformulating} its underlying mathematical principle. 

The effectiveness of cosine similarity hinges on an angular interpretation of semantic space. This motivates a fundamental question: \textbf{beyond angular alignment, do other relational structures in embeddings---such as consistent ordinal patterns across dimensions---encode meaningful semantic information?} 

To explore this, we revisit the mathematical foundation of similarity measurement. The Cauchy-Schwarz inequality, which underpins cosine similarity, is not the tightest possible bound for the dot product. We derive a strictly tighter upper bound based on the Rearrangement Inequality, which naturally leads to a new similarity measure, \textit{recos}. This measure relaxes the requirement for perfect similarity from strict linear dependence (as in cosine similarity) to ordinal concordance, suggesting a wider capture range for structured, potentially non-linear relationships.

In this work, we: (1) establish a novel hierarchy of inequalities for the dot product, formalizing this new measure and related metrics; (2) theoretically characterize their properties and relationships; (3) conduct a comprehensive evaluation across diverse embedding models and STS benchmarks, showing our measure consistently outperforms cosine similarity. Our findings reveal that monotonic patterns in embeddings provide a signal complementary to angular alignment, constitute a meaningful and complementary signal for analyzing semantic representations.

\section{Methods}

\subsection{Definitions and Notations}

\begin{Definition}[Similar Vectors]\label{def_similar_vectors}
  Let $\mathbf{u} = (u_1, \dots, u_d)$ and $\mathbf{v} = (v_1, \dots, v_d)$ be vectors in $\mathbb{R}^d$. They are said to be \textbf{similar vectors} if they are similarly ordered; that is, if for all $1 \le i,j \le d$, the inequality $\quad (u_i - u_j)(v_i - v_j) \geq 0$ holds. This is equivalent to the non-existence of any discordant pair $\left( {i,j} \right)$, defined by the property that $\left( {{u_i} - {u_j}} \right)$ and $\left( {{v_i} - {v_j}} \right)$ have opposite signs.
\end{Definition}

\begin{Definition}[Discordant Vectors]\label{def_discordant_vectors}
  Let $\mathbf{u} = (u_1, \dots, u_d)$ and $\mathbf{v} = (v_1, \dots, v_d)$ be vectors in $\mathbb{R}^d$. They are said to be \textbf{discordant vectors} if they are \textit{oppositely ordered}; that is, if for all $1 \le i,j \le d$, the inequality $(u_i - u_j)(v_i - v_j) \leq 0$ holds. This is equivalent to the non-existence of any \textit{concordant pair} $(i, j)$, defined by the property that $(u_i - u_j)$ and $(v_i - v_j)$ have the same sign.
\end{Definition}

To ground these definitions, consider an example where three experts score four candidates on a scale from 0 to 10. Each expert’s scores form a vector in $\mathbb{R}^4$:
\begin{itemize}
    \item Expert 1: $\mathbf{e}_1 = (1, 5.5, 2, 4)$
    \item Expert 2: $\mathbf{e}_2 = (2, 6.0, 3, 5)$
    \item Expert 3: $\mathbf{e}_3 = (9, 4.5, 8, 6)$
\end{itemize}

Experts 1 and 2 produce similar evaluations: although their absolute scores differ, they agree on the relative ranking of all candidates (Candidate 2 $\succ$ Candidate 4 $\succ$ Candidate 3 $\succ$ Candidate 1). In contrast, Expert 3's scores are largely reversed relative to Expert 1; for instance, Candidate 1 receives the lowest score from Expert 1 but the highest from Expert 3. This demonstrates a dissimilar assessment.

Our notion of similarity captures the concordance of two vectors, i.e., whether they induce the same ordering of components, which aligns with the intuitive idea of ``agreeing in relative judgment''. This concept possesses three important properties:
\begin{itemize}
  \item Similarity differs from identity. $\mathbf{e}_1$ and $\mathbf{e}_2$ are similar but not the same, thus they can be numerically distinct.
  \item Similarity is not restricted to linear (or proportional) relationships. In real-world evaluations, perfectly proportional score vectors are exceedingly rare.
  \item The essential characteristic of similarity is ordinal concordance rather than metric alignment.
\end{itemize}

Cosine similarity, defined as the cosine of the angle between two vectors, demands strict proportionality for a perfect score. This corresponds to the special case where all component pairs satisfy ${u_i} = k{v_i}$ for some scalar $k$. Such a stringent condition is seldom met in practice; for instance, no expert's ratings are likely to be exactly proportional to $\mathbf{e}_1$. In contrast, our concordance-based definition captures a broader, more realistic form of agreement, which often reflects ``similar opinions'' in comparative assessment tasks. 

\begin{Definition}[Vector Ordering]
Let $\mathbf{u}, \mathbf{v} \in \mathbb{R}^d$ with $\mathbf{u} \cdot \mathbf{v} \neq 0$. We define:
\begin{itemize}
    \item $\mathbf{u}^\uparrow$ denotes $\mathbf{u}$ sorted in non-decreasing order.
    \item $\mathbf{v}^\uparrow$ denotes $\mathbf{v}$ sorted in non-decreasing order.
    \item $\mathbf{v}^\downarrow$ denotes $\mathbf{v}$ sorted in non-increasing order.
    \item $\mathbf{v}^{\updownarrow} = \begin{cases}
        \mathbf{v}^\uparrow, & \text{if } \mathbf{u} \cdot \mathbf{v} > 0, \\[4pt]
        \mathbf{v}^\downarrow, & \text{if } \mathbf{u} \cdot \mathbf{v} < 0.
    \end{cases}$
\end{itemize}
\end{Definition}

For the remainder of this work, we systematically exclude the case where $\mathbf{u} \cdot \mathbf{v} = 0$ from our theoretical analysis and experimental evaluations. This exclusion is justified by two primary considerations. First, in practical high-dimensional embedding spaces, exact orthogonality between non-zero vectors constitutes a measure-zero event, making it a pathological scenario with limited practical relevance to real-world applications. Second, and more fundamentally, when $\mathbf{u} \cdot \mathbf{v} = 0$, all normalized similarity measures—including both standard cosine similarity and the proposed similarity metrics discussed in this work—yield identical zero values, thereby eliminating any discriminative power between different normalization approaches. This convention aligns with established practices in similarity measurement literature \cite{taffeet.ElementaryMathematicalTheory1958}, where zero-correlation boundary cases are typically excluded to maintain theoretical consistency and practical utility.

\subsection{New Inequality Series for Dot Product}

Let $\mathbf{u}, \mathbf{v} \in \mathbb{R}^d$. The Cauchy--Schwarz inequality states that
\begin{equation}
  \left| \mathbf{u} \cdot \mathbf{v} \right| \le \| \mathbf{u} \| \| \mathbf{v} \|,
\end{equation}
where $\mathbf{u} \cdot \mathbf{v} = \sum_{i=1}^d u_i v_i$ denotes the dot product and $\|\mathbf{u}\| = \sqrt{\sum_{i=1}^d u_i^2}$ is the Euclidean norm. Equality holds if and only if $\mathbf{u}$ and $\mathbf{v}$ are linearly dependent.

We sharpen this classical bound and introduce a novel series of inequalities.

\begin{figure*}
  \centering
  \includegraphics[width=0.9\linewidth]{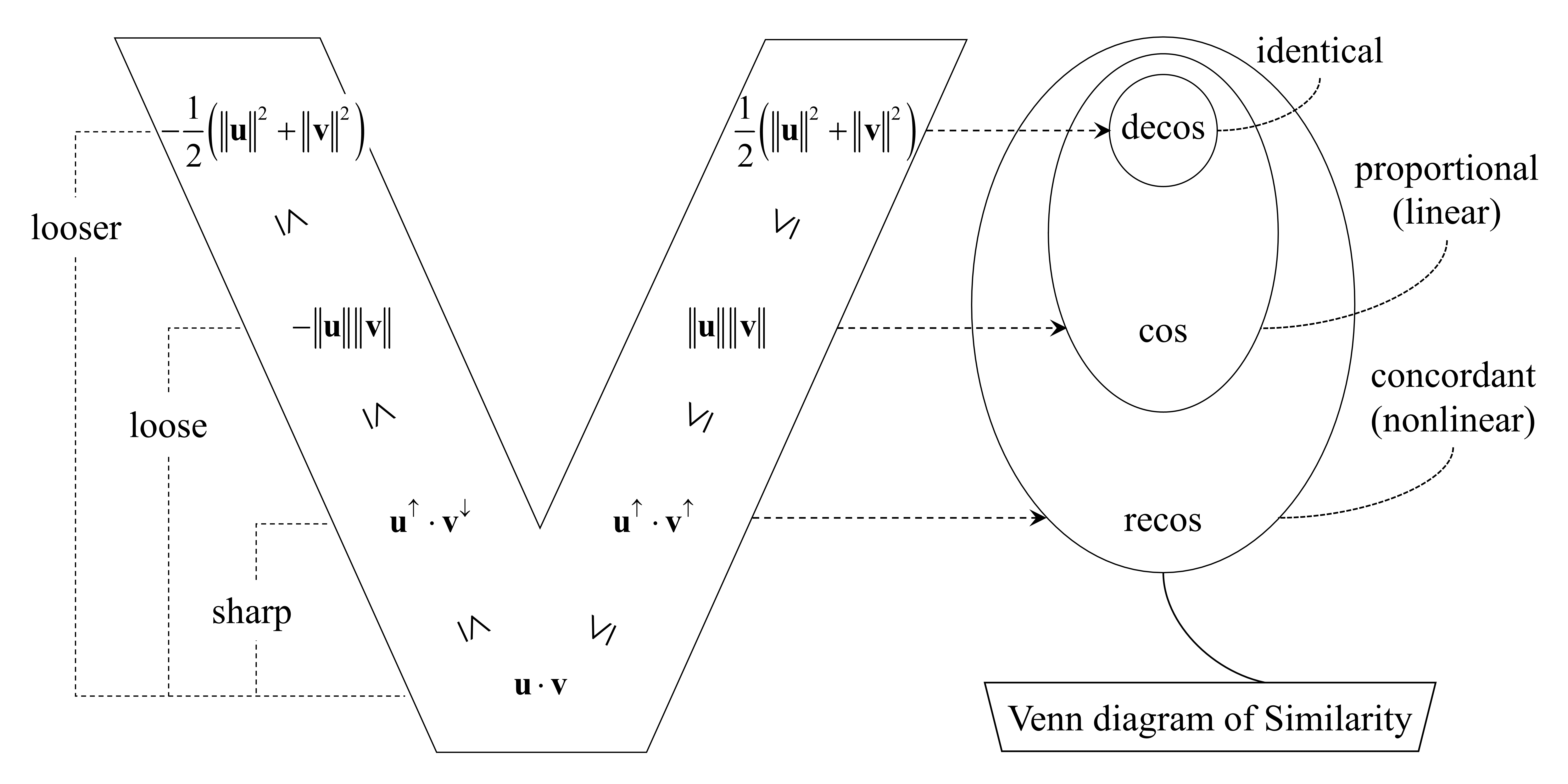}
  \caption{Similarity metrics derived from different bounds and their effective capture ranges. A tighter bound (smaller denominator) leads to a more permissive metric with a wider capture range.}
  \label{fig_Framework}
\end{figure*}

\begin{Theorem}[Chain of Inequalities]\label{thm_inequality_series}
  For vectors $\mathbf{u}$ and $\mathbf{v}$ we have
    \begin{equation*}
    \left| {{\mathbf{u}} \cdot {\mathbf{v}}} \right| \le \left|  {{\mathbf{u}}^ \uparrow \cdot {\mathbf{v}}^\updownarrow}\right| \le \left\| {\mathbf{u}} \right\|\left\| {\mathbf{v}} \right\| \le \frac{1}{2}\left( {{{\left\| {\mathbf{u}} \right\|}^2} + {{\left\| {\mathbf{v}} \right\|}^2}} \right),
  \end{equation*}
  and the equality conditions are as follows:
    \begin{itemize}
    \item $\left| {{\mathbf{u}} \cdot {\mathbf{v}}} \right| = \left|  {{\mathbf{u}}^ \uparrow \cdot {\mathbf{v}}^\updownarrow}\right|$ if and only if $\mathbf{u} \cdot \mathbf{v} > 0$ and $\mathbf{u},\mathbf{v}$ are similar vectors, or $\mathbf{u} \cdot \mathbf{v} < 0$ and they are discordant vectors.  
    \item $\left| {{\mathbf{u}} \cdot {\mathbf{v}}} \right| = \left\| {\mathbf{u}} \right\|\left\| {\mathbf{v}} \right\|$ if and only if  $\mathbf{u}$ and $\mathbf{v}$ are linearly dependent vectors, i.e., there exists a scalar $k$ such that ${\mathbf{v}} = k{\mathbf{u}}$.
    \item $\left| {{\mathbf{u}} \cdot {\mathbf{v}}} \right| = \frac{1}{2}\left( {{{\left\| \mathbf{u} \right\|}^2} + {{\left\| \mathbf{v} \right\|}^2}} \right)$ if and only if $\mathbf{u}$ and $\mathbf{v}$ are identical or opposite vectors, i.e.,$\mathbf{u} =  \pm \mathbf{v}$.
    \item $\left|  {{\mathbf{u}}^ \uparrow \cdot {\mathbf{v}}^\updownarrow}\right| = \left\| {\mathbf{u}} \right\|\left\| {\mathbf{v}} \right\|$ if and only if there exists a scalar $k$ and a permutation matrix $P$ such that $\mathbf{v} = k P \mathbf{u}$, with $\operatorname{sgn} \left( {\mathbf{u} \cdot \mathbf{v}} \right) = \operatorname{sgn} \left( k \right)$.
    \item  $\left|  {{\mathbf{u}}^ \uparrow \cdot {\mathbf{v}}^\updownarrow}\right| = \frac{1}{2}\left( {{{\left\| {\mathbf{u}} \right\|}^2} + {{\left\| {\mathbf{v}} \right\|}^2}} \right)$ if and only if there exists a scalar $k = \pm 1$ and a permutation matrix $P$ such that $\mathbf{v} = k P \mathbf{u}$, with $\operatorname{sgn} \left( {\mathbf{u} \cdot \mathbf{v}} \right) = \operatorname{sgn} \left( k \right)$.
    \item $\left\| {\mathbf{u}} \right\|\left\| {\mathbf{v}} \right\| = \frac{1}{2}\left( {{{\left\| {\mathbf{u}} \right\|}^2} + {{\left\| {\mathbf{v}} \right\|}^2}} \right)$ if and only if $\left\| {\mathbf{u}} \right\| = \left\| {\mathbf{v}} \right\|$.
  \end{itemize}
\end{Theorem}

Proofs of the above theorem and subsequent corollaries are provided in the appendix.

Theorem~\ref{thm_inequality_series} establishes a hierarchy of upper bounds for the dot product. The new term $\left| \mathbf{u}^{\uparrow} \cdot \mathbf{v}^{\updownarrow} \right|$ is the \emph{tightest} bound, followed by the Cauchy--Schwarz bound $\| \mathbf{u} \| \| \mathbf{v} \|$, and finally the bound $\frac{1}{2}(\| \mathbf{u} \|^2 + \| \mathbf{v} \|^2)$. This hierarchy naturally suggests alternative normalization schemes for the dot product, moving beyond the conventional cosine similarity.

Each bound induces a distinct similarity measure with a different \emph{saturation condition} for achieving a maximum score of 1. The tightest bound leads to a measure that saturates under the broadest condition (ordinal concordance), while the loosest bound requires the strictest condition (vector identity/anti-identity). Cosine similarity, based on the intermediate Cauchy--Schwarz bound, saturates under linear dependence.

\subsection{Similarity Metrics Derived from Different Bounds}
\label{sec:metrics}

The inequality series in Theorem~\ref{thm_inequality_series} provides three distinct bounds, each leading to a different normalization of the dot product. We define the corresponding similarity metrics: $\mathrm{recos}$, $\cos$, and $\mathrm{decos}$, as illustrated in Figure~\ref{fig_Framework}.

\begin{Definition}[$\mathrm{recos}$] \label{def_recos}
  The Rearrangement-inequality-based Cosine Similarity ($\mathrm{recos}$) is defined as:
  \begin{equation}
    \mathrm{recos}(\mathbf{u}, \mathbf{v}) = \frac{\mathbf{u} \cdot \mathbf{v}}{\left| \mathbf{u}^{\uparrow} \cdot \mathbf{v}^{\updownarrow} \right|}.
  \end{equation}
\end{Definition}

\begin{Definition}[$\cos$] \label{def_cos}
  The Cosine Similarity is defined as \cite{manningIntroductionInformationRetrieval2008}:
  \begin{equation}
    \cos(\mathbf{u}, \mathbf{v}) = \frac{\mathbf{u} \cdot \mathbf{v}}{\| \mathbf{u} \| \| \mathbf{v} \|}.
  \end{equation}
\end{Definition}

\begin{Definition}[$\mathrm{decos}$] \label{def_decos}
  The Degenerated Cosine Similarity ($\mathrm{decos}$), based on the inequality of arithmetic and quadratic means, is defined as:
  \begin{equation}
    \mathrm{decos}(\mathbf{u}, \mathbf{v}) = \frac{\mathbf{u} \cdot \mathbf{v}}{\frac{1}{2} \left( \| \mathbf{u} \|^2 + \| \mathbf{v} \|^2 \right)}.
  \end{equation}
  We term it ``degenerated'' because it primarily measures near-identity relationships rather than general similarity.
\end{Definition}

The $\mathrm{decos}$ metric is closely related to the classical Tanimoto coefficient \cite{taffeet.ElementaryMathematicalTheory1958}.
\begin{Definition}[Tanimoto Similarity] \label{def_tan}
  The Tanimoto similarity is defined as:
  \begin{equation}
    \mathrm{tan}(\mathbf{u}, \mathbf{v}) = \frac{\mathbf{u} \cdot \mathbf{v}}{\| \mathbf{u} \|^2 + \| \mathbf{v} \|^2 - \mathbf{u} \cdot \mathbf{v}}.
  \end{equation}
\end{Definition}
Assuming non-negative correlation ($\mathbf{u} \cdot \mathbf{v} > 0$), the inequality $\mathbf{u} \cdot \mathbf{v} \le \frac{1}{2}(\|\mathbf{u}\|^2 + \|\mathbf{v}\|^2)$ implies $\mathbf{u} \cdot \mathbf{v} \le \|\mathbf{u}\|^2 + \|\mathbf{v}\|^2 - \mathbf{u} \cdot \mathbf{v}$. Thus, $\mathrm{tan}$ shares the same theoretical bounds as $\mathrm{decos}$. Moreover, for $\mathrm{tan}(\mathbf{u}, \mathbf{v}) \in [0,1]$, the two are related by a strictly monotonic bijection:
\begin{equation}
  \mathrm{decos}(\mathbf{u}, \mathbf{v}) = \frac{2 \cdot \mathrm{tan}(\mathbf{u}, \mathbf{v})}{1 + \mathrm{tan}(\mathbf{u}, \mathbf{v})}.
\end{equation}
Since this transformation preserves rank order, our subsequent analysis focuses on $\mathrm{decos}$ without loss of generality.

The three metrics exhibit distinct theoretical properties, as formalized below.
\begin{Corollary}[Bounds and Saturation Conditions]\label{re_de_cos_properties}
  For vectors $\mathbf{u}$, $\mathbf{v}$, we have 
  \begin{itemize}
    \item $\left| {\mathrm{recos}\left( {{\mathbf{u}},{\mathbf{v}}} \right)} \right| \le 1$ and the equality holds if and only if $\mathbf{u} \cdot \mathbf{v} > 0$ and $\mathbf{u},\mathbf{v}$ are similar vectors, or $\mathbf{u} \cdot \mathbf{v} < 0$ and they are discordant vectors.
    \item $\left| {\cos\left( {{\mathbf{u}},{\mathbf{v}}} \right)} \right| \le 1$ and the equality holds if and only if  $\mathbf{u}$ and $\mathbf{v}$ are linearly dependent vectors, i.e., there exists a scalar $k$ such that ${\mathbf{v}} = k{\mathbf{u}}$.
    \item $\left| {\mathrm{decos}\left( {{\mathbf{u}},{\mathbf{v}}} \right)} \right| \le 1$ and the equality holds if and only if $\mathbf{u}$ and $\mathbf{v}$ are identical or opposite vectors, i.e.,$\mathbf{u} =  \pm \mathbf{v}$.
  \end{itemize}
\end{Corollary}

Corollary~\ref{re_de_cos_properties} reveals the hierarchy of \emph{strictness} among the saturation conditions. Achieving a similarity score of $1$ is most restrictive for $\mathrm{decos}$ (requiring vector identity/anti-identity), less restrictive for $\cos$ (requiring linear dependence), and least restrictive for $\mathrm{recos}$ (requiring only a monotonic relationship). Consequently, $\mathrm{recos}$ possesses the \emph{widest effective capture range}, enabling it to saturate---and thus accurately reflect high similarity---for a broader class of vector relationships, including non-linear monotonic associations common in semantic spaces.

\begin{Corollary}[Metric Hierarchy]\label{metric_hierarchy}
  For vectors $\mathbf{u}$, $\mathbf{v}$, we have 
    \begin{equation}
    \left| {\mathrm{decos}\left( {{\mathbf{u}},{\mathbf{v}}} \right)} \right| \le \left| {\cos\left( {{\mathbf{u}},{\mathbf{v}}} \right)} \right| \le \left| {\mathrm{recos}\left( {{\mathbf{u}},{\mathbf{v}}} \right)} \right|,
  \end{equation}
  and the equality conditions are as follows:
  \begin{itemize}
    \item $\left| {\mathrm{recos}\left( {{\mathbf{u}},{\mathbf{v}}} \right)} \right| = \left| {\cos\left( {{\mathbf{u}},{\mathbf{v}}} \right)} \right|$ if and only if there exists a scalar $k$ and a permutation matrix $P$ such that $\mathbf{v} = k P \mathbf{u}$, with $\operatorname{sgn} \left( {\mathbf{u} \cdot \mathbf{v}} \right) = \operatorname{sgn} \left( k \right)$.
    \item $\left| {\mathrm{recos}\left( {{\mathbf{u}},{\mathbf{v}}} \right)} \right| = \left| {\mathrm{decos}\left( {{\mathbf{u}},{\mathbf{v}}} \right)} \right|$ if and only if there exists a scalar $k = \pm 1$ and a permutation matrix $P$ such that $\mathbf{v} = k P \mathbf{u}$, with $\operatorname{sgn} \left( {\mathbf{u} \cdot \mathbf{v}} \right) = \operatorname{sgn} \left( k \right)$.
    \item $\left| {\cos\left( {{\mathbf{u}},{\mathbf{v}}} \right)} \right| = \left| {\mathrm{decos}\left( {{\mathbf{u}},{\mathbf{v}}} \right)} \right|$ if and only if $\left\| {\mathbf{u}} \right\| = \left\| {\mathbf{v}} \right\|$.
  \end{itemize}
\end{Corollary}

The first equality condition shows that $\mathrm{recos}$ reduces to $\mathrm{cos}$ when one vector is a permuted and scaled version of the other—a condition strictly more general than linear dependence alone. This theoretical hierarchy ($\mathrm{decos} \subset \cos \subset \mathrm{recos}$ in terms of capture range) indicates that $\mathrm{recos}$ captures the broadest class of similarity relationships, particularly those beyond simple linear correlation. 

A common practice in similarity computation is to normalize input vectors to unit length. Under this precondition, the relationship between $\mathrm{decos}$ and $\cos$ simplifies, as formalized below.

\begin{Corollary}[Norm Identity]\label{cor_normalized_vectors}
    For unit-norm vectors where $\| \mathbf{u} \| = \| \mathbf{v} \| = 1$, the following identity holds:
    \begin{equation}
        \mathrm{decos}(\mathbf{u}, \mathbf{v}) = \cos(\mathbf{u}, \mathbf{v}).
    \end{equation}
\end{Corollary}
For unit-norm vectors, the denominators of $\mathrm{cos}$ and $\mathrm{decos}$ coincide, leading to the identity stated in Corollary~\ref{cor_normalized_vectors}. It establishes that $\mathrm{decos}$ and $\cos$ are \emph{identical} for unit-norm vectors. Consequently, in standard pipelines where cosine similarity is applied to normalized embeddings, replacing $\cos$ with $\mathrm{decos}$ yields no difference.

This equivalence, however, underscores a key limitation of both metrics: they reduce to the same measure when magnitude information is discarded. In contrast, $\mathrm{recos}$ remains distinct even for normalized vectors because its denominator, $\left| \mathbf{u}^{\uparrow} \cdot \mathbf{v}^{\updownarrow} \right|$, depends on the ordinal structure of vector components, not merely their norms. This property allows $\mathrm{recos}$ to capture similarity based on monotonic relationships, a feature preserved under normalization. Thus, $\mathrm{recos}$ provides a strictly more general similarity measure that does not collapse to standard cosine similarity under the common unit-norm constraint.

\subsection{An Illustrative Example}

The normalization bound applied to the dot product of two vectors $\mathbf{u}$ and $\mathbf{v}$ fundamentally determines the sensitivity and scope—or \textit{capture range}—of the resulting similarity measure. An inappropriate bound can lead to systematic \textit{underestimation} of true similarity, particularly when the vector relationship deviates from the measure's assumed structure. To illustrate this, we extend the previous expert evaluation example with three additional raters (Experts 4, 5, and 6) to demonstrate how different similarity measures respond to distinct relationship types:
\begin{itemize}
    \item Expert 1: $\mathbf{e}_1 = (1, 5.5, 2, 4)$
    \item Expert 4: $\mathbf{e}_4 = (2, 5.5, 1, 4)$ (genuine dissimilarity)
    \item Expert 5: $\mathbf{e}_5 = 1.225 \times \mathbf{e}_1$ (linear similarity)
    \item Expert 6: $\mathbf{e}_6 = (1, 8.5, 2, 4)$ (nonlinear, order-preserving similarity)
\end{itemize}
Here, $\mathbf{e}_4$ and $\mathbf{e}_1$ reflect genuine disagreement, as they hold opposing preferences between the third and fourth candidates, while $\mathbf{e}_5$ and $\mathbf{e}_6$ represent two distinct types of similarity relationships with $\mathbf{e}_1$.

The calculated similarities reveal the limitations of each metric:
\begin{itemize}
    \item $\mathrm{decos}(\mathbf{e}_1, \mathbf{e}_5) \approx \mathrm{decos}(\mathbf{e}_1, \mathbf{e}_4) \approx 0.98$. The $\mathrm{decos}$ metric fails to distinguish between a linearly related vector ($\mathbf{e}_5$) and a genuinely dissimilar one ($\mathbf{e}_4$). This illustrates its narrow capture range: it inherently interprets any deviation from perfect identity as dissimilarity, even when a strong linear relationship exists.
    \item $\cos(\mathbf{e}_1, \mathbf{e}_5) = 1.00 > \cos(\mathbf{e}_1, \mathbf{e}_4) \approx 0.98$. Unlike $\mathrm{decos}$, $\cos$ correctly identifies the perfect linear relationship in $\mathbf{e}_5$, demonstrating its suitability for capturing linear associations.
    \item However, $\cos(\mathbf{e}_1, \mathbf{e}_6) \approx \cos(\mathbf{e}_1, \mathbf{e}_4) \approx 0.98$. Here, $\cos$ cannot differentiate between the genuinely dissimilar pair $(\mathbf{e}_1, \mathbf{e}_4)$ and the nonlinear but order-preserving pair $(\mathbf{e}_1, \mathbf{e}_6)$. It assigns them equal similarity scores, unable to discern whether a low score indicates a true lack of correlation or merely a consistent, nonlinear association.
    \item $\mathrm{recos}(\mathbf{e}_1, \mathbf{e}_6) = 1.00 > \mathrm{recos}(\mathbf{e}_1, \mathbf{e}_4) \approx 0.98$. In contrast, $\mathrm{recos}$ recognizes the perfect ordinal agreement in $(\mathbf{e}_1, \mathbf{e}_6)$ and assigns a maximal score, while correctly assigning a lower score to the genuinely dissimilar pair $(\mathbf{e}_1, \mathbf{e}_4)$. This shows that $\mathrm{recos}$ effectively captures nonlinear, concordant relationships without conflating them with true dissimilarity.
\end{itemize}

This example highlights the core issue of underestimation through an intuitive analogy: $\mathrm{decos}$ is akin to a tool that can only verify perfect identity; when assessing general similarity, it cannot determine whether a low score stems from genuine dissimilarity or from a strongly related but non-identical structure. Likewise, $\mathrm{cos}$ is calibrated for linear relationships; when presented with a broader, concordant but nonlinear association, it yields an underestimated score, unable to discern whether the deviation indicates true dissimilarity or merely a nonlinear yet consistent relationship. The $\mathrm{recos}$ metric, by employing a tighter normalization bound derived from the Rearrangement Inequality, avoids this pitfall. Its wider capture range enables it to accurately reflect strong, order-preserving relationships without underestimation.

Different metrics reflect different inductive biases. $\mathrm{decos}$ is highly sensitive to magnitude differences, $\cos$ captures angular alignment (linear correlation), while $\mathrm{recos}$ responds to ordinal concordance. The example is illustrative and uses low-dimensional vectors; the practical relevance of ordinal concordance in high-dimensional embedding spaces is an empirical question we investigate next.

\subsection{Computational Complexity Analysis}

The computation of $\mathrm{recos}$ entails a sorting step to obtain $\mathbf{u}^{\uparrow}$ and $\mathbf{v}^{\updownarrow}$, incurring a time complexity of $O(d \log d)$ compared to the $O(d)$ of standard cosine similarity. For typical embedding dimensions ($d = 128$--$1024$), the practical overhead of sorting is often negligible for single or batch comparisons due to highly optimized sorting algorithms and hardware. However, in large-scale retrieval scenarios involving billions of vectors, this logarithmic factor increase can become significant.

\begin{figure*}[htb]
\begin{center}
  \includegraphics[width=0.9\linewidth]{{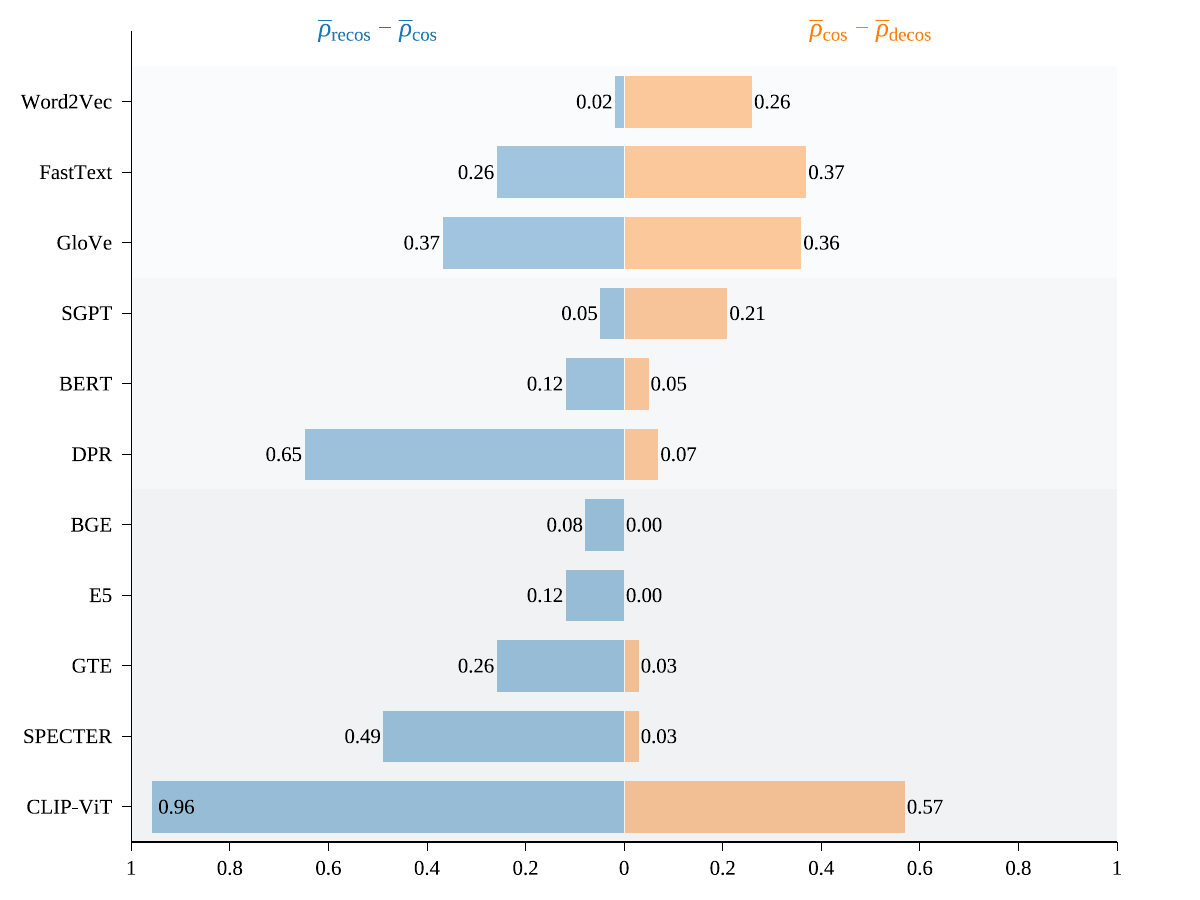}}
  \caption{Performance gains of $\mathrm{recos}$ over $\mathrm{cos}$ and $\mathrm{cos}$ over $\mathrm{decos}$}
  \label{fig_performance_improvement}
\end{center}
\end{figure*}

\section{Experimental Evaluation}

\subsection{Datasets and Metrics}

We evaluate the proposed similarity metric $\mathrm{recos}$ on standard Semantic Textual Similarity (STS) benchmarks, which provide the most direct and appropriate testbed for assessing sentence-level semantic comparisons. The core objective of STS—to assign a graded score reflecting the degree of semantic equivalence between two sentences—aligns precisely with the fundamental purpose of a similarity metric. Therefore, performance on these tasks offers the most straightforward evidence of a metric's ability to approximate human semantic judgment.

Our evaluation encompasses seven widely-used datasets: STS12, STS13, STS14, STS15, STS16 (from the SemEval shared tasks 2012-2016), STS-B (STSbenchmark), and SICK-R. These datasets provide high-quality, human-annotated similarity scores across diverse domains, ensuring a comprehensive and rigorous assessment. 

Performance is measured by Spearman's rank correlation $\rho$ (conventionally reported as $\rho \times 100$) between the similarity score computed from sentence embeddings and the human gold-standard scores. This established metric enables a direct and fair comparison with prior work on sentence representations and similarity measures, grounding our analysis in the prevailing research context.

\subsection{Models and Setup}
We conduct a comprehensive evaluation across 11 diverse pre-trained language models, categorized into three groups: (1) static dense word embeddings (Word2Vec, FastText, GloVe); (2) contextualized embeddings (BERT, SGPT, DPR); (3) universal text embeddings (BGE, E5, GTE, SPECTER, CLIP-ViT).The models span diverse architectures and training paradigms, ensuring a comprehensive assessment.

For each model, we compute embeddings for all sentence pairs and calculate similarity using three metrics: the proposed $\mathrm{recos}$, standard cosine similarity ($\cos$), and dot-product similarity ($\mathrm{decos}$). This yields $11 \times 7 = 77$ unique evaluation settings.

All implementations, including data preprocessing, embedding generation, and similarity computation, are available as executable Jupyter notebooks at \url{https://github.com/byaxb/recos}. The complete environment configuration, model specifications, and step-by-step instructions to exactly reproduce all reported results are provided in the appendix.

\subsection{Results}
The full experimental results are provided in Table~\ref{tab:detailed_results}, and the relative performance gains are summarized visually in Figure~\ref{fig_performance_improvement}.

\begin{table*}[t]
\small
\setlength{\tabcolsep}{6pt}
\centering
\caption{Performance comparison of different similarity metrics ($\mathrm{decos}$, $\mathrm{cos}$ and $\mathrm{recos}$) across various pre-trained language models on the STS benchmark. The best result for each model is highlighted in bold.}
\label{tab:detailed_results}
\begin{tabular}{r l c c c c c c c c}
\toprule
Model & Method & STS12 & STS13 & STS14 & STS15 & STS16 & STS-B & SICK-R & Avg. \\
\midrule
Word2Vec & decos & 53.07 & 69.63 & 65.42 & 75.19 & 67.09 & 64.30 & 57.83 & 64.65 \\
Word2Vec & cos & 52.78 & 70.12 & 65.74 & 75.54 & \textbf{67.69} & 64.51 & 58.00 & 64.91 \\
Word2Vec & recos & \textbf{52.80} & \textbf{70.14} & \textbf{65.77} & \textbf{75.55} & \textbf{67.69} & \textbf{64.53} & \textbf{58.06} & \textbf{64.93} \\
\addlinespace
FastText & decos & 58.41 & 71.22 & 64.69 & 73.49 & 63.95 & 63.70 & 56.48 & 64.56 \\
FastText & cos & \textbf{58.55} & 71.66 & 65.23 & 74.03 & 64.49 & 63.99 & 56.62 & 64.94 \\
FastText & recos & \textbf{58.55} & \textbf{71.98} & \textbf{65.58} & \textbf{74.37} & \textbf{64.79} & \textbf{64.32} & \textbf{56.71} & \textbf{65.19} \\
\addlinespace
GloVe & decos & 57.35 & 70.75 & 59.84 & 70.42 & 63.41 & 50.41 & 55.31 & 61.07 \\
GloVe & cos & 57.49 & 70.99 & 60.70 & 70.85 & 63.84 & 50.74 & 55.42 & 61.43 \\
GloVe & recos & \textbf{57.68} & \textbf{71.37} & \textbf{61.44} & \textbf{71.09} & \textbf{64.14} & \textbf{51.29} & \textbf{55.62} & \textbf{61.80} \\
\addlinespace
BERT & decos & \textbf{72.54} & 77.93 & 73.73 & 79.43 & 74.43 & 76.71 & 73.53 & 75.47 \\
BERT & cos & 72.52 & 78.05 & 73.86 & 79.48 & 74.47 & 76.74 & 73.54 & 75.52 \\
BERT & recos & 72.52 & \textbf{78.30} & \textbf{74.01} & \textbf{79.66} & \textbf{74.57} & \textbf{76.82} & \textbf{73.64} & \textbf{75.65} \\
\addlinespace
SGPT & decos & 66.27 & 69.96 & 63.50 & 75.29 & 71.60 & 72.69 & 67.68 & 69.57 \\
SGPT & cos & 66.44 & 70.13 & 63.71 & 75.54 & 71.75 & 72.89 & 68.00 & 69.78 \\
SGPT & recos & \textbf{66.54} & \textbf{70.15} & \textbf{63.79} & \textbf{75.58} & 71.75 & \textbf{72.92} & \textbf{68.08} & \textbf{69.83} \\
\addlinespace
DPR & decos & 47.92 & 63.88 & 53.16 & 67.21 & 66.55 & 58.46 & 62.39 & 59.94 \\
DPR & cos & 47.96 & 63.97 & 53.22 & 67.25 & 66.70 & 58.53 & 62.45 & 60.01 \\
DPR & recos & \textbf{49.09} & \textbf{64.11} & \textbf{53.74} & \textbf{67.40} & \textbf{67.28} & \textbf{59.22} & \textbf{63.78} & \textbf{60.66} \\
\addlinespace
E5 & decos & 60.42 & 70.00 & 65.30 & 74.92 & 77.43 & 73.53 & 68.86 & 70.07 \\
E5 & cos & 60.42 & 70.00 & 65.30 & 74.92 & 77.43 & 73.53 & 68.86 & 70.07 \\
E5 & recos & \textbf{60.70} & \textbf{70.16} & \textbf{65.44} & \textbf{75.05} & \textbf{77.44} & \textbf{73.57} & \textbf{68.93} & \textbf{70.18} \\
\addlinespace
BGE & decos & 49.97 & \textbf{32.19} & 33.64 & 35.01 & 63.81 & 65.61 & 57.94 & 48.31 \\
BGE & cos & 49.97 & \textbf{32.19} & 33.64 & 35.01 & 63.81 & 65.61 & 57.94 & 48.31 \\
BGE & recos & \textbf{50.28} & 31.88 & \textbf{33.82} & \textbf{35.12} & \textbf{63.94} & \textbf{65.71} & \textbf{57.99} & \textbf{48.39} \\
\addlinespace
GTE & decos & 74.81 & 86.63 & 78.74 & 85.38 & 83.10 & 85.42 & 75.72 & 81.40 \\
GTE & cos & 74.82 & 86.64 & 78.76 & 85.39 & 83.10 & 85.42 & 75.72 & 81.41 \\
GTE & recos & \textbf{75.00} & \textbf{86.77} & \textbf{78.98} & \textbf{85.50} & \textbf{83.17} & \textbf{85.45} & \textbf{75.72} & \textbf{81.51} \\
\addlinespace
SPECTER & decos & 62.45 & 58.62 & 54.82 & 62.49 & 64.27 & 61.23 & 56.36 & 60.03 \\
SPECTER & cos & 62.49 & 58.70 & 54.87 & 62.54 & 64.28 & 61.26 & 56.39 & 60.08 \\
SPECTER & recos & \textbf{63.01} & \textbf{59.70} & \textbf{55.45} & \textbf{63.10} & \textbf{64.47} & \textbf{61.50} & \textbf{56.71} & \textbf{60.56} \\
\addlinespace
CLIP-ViT & decos & 76.71 & 63.10 & 57.15 & 65.58 & 72.08 & 64.80 & 69.89 & 67.04 \\
CLIP-ViT & cos & 76.92 & 63.96 & 57.98 & 66.19 & 72.41 & 65.29 & 70.53 & 67.61 \\
CLIP-ViT & recos & \textbf{77.39} & \textbf{64.81} & \textbf{59.34} & \textbf{67.55} & \textbf{72.78} & \textbf{66.53} & \textbf{71.61} & \textbf{68.57} \\
\bottomrule
\end{tabular}
\end{table*}

\paragraph{Overall Performance}
The experimental results in Table~\ref{tab:detailed_results} show that $\mathrm{recos}$ consistently matches or exceeds the performance of $\cos$ across the majority of evaluated models and datasets. Aggregating over the 77 evaluation settings, $\mathrm{recos}$ outperforms $\cos$ in 71 cases (92.2\%), underperforms in 1 instance, and matches performance in 5. Similarly, $\cos$ surpasses $\mathrm{decos}$ in 58 cases, with performance ties occurring primarily for models producing length-normalized embeddings (e.g., BGE, E5), where theoretical bounds coincide as established in Corollary~\ref{cor_normalized_vectors}.

The mean performance further confirms this hierarchical relationship: $\mathrm{recos}$ achieves a micro-average score of 66.12, exceeding both $\cos$ (65.83) and $\mathrm{decos}$ (65.65). The average improvement of $\mathrm{recos}$ over $\cos$ is 0.29 points. While this improvement is modest in absolute terms, its consistency across diverse models is noteworthy.

\paragraph{Group-wise Analysis}
The performance improvement provided by $\mathrm{recos}$ exhibits a clear correlation with model sophistication. Static embeddings (Word2Vec, FastText, GloVe) show modest gains (0.02--0.37 points). Contextualized embeddings (BERT, SGPT, DPR) demonstrate more pronounced improvements (0.05--0.65 points), as $\mathrm{recos}$ better leverages their dynamic representations. Universal embeddings (BGE, E5, GTE, SPECTER) exhibit the most substantial gains (0.08--0.96 points), indicating that $\mathrm{recos}$ provides the largest performance lift for models with complex representation spaces.

\paragraph{Notable Improvements}
The most dramatic improvements of $\mathrm{recos}$ occur for models with specialized training objectives. CLIP-ViT shows exceptional gains of +0.96 points (peaking at +1.36 on STS14 and STS15), while DPR and SPECTER achieve improvements of +0.65 and +0.49 points respectively. These results indicate that $\mathrm{recos}$ particularly benefits models whose representation spaces diverge from standard textual similarity assumptions, effectively mitigating the underestimation inherent in $\cos$ for specialized embeddings.

\begin{table*}
\centering
\small
\caption{Statistical comparison of \textsc{recos} versus \textsc{cos} across 77 model-dataset pairs (11 models $\times$ 7 datasets). All tests indicate statistically significant improvements with $p < 0.001$.}
\label{tab:statistical_results}
\begin{tabular}{@{}lcccc@{}}
\toprule
\textbf{Test} & \textbf{Statistic} & \textbf{$p$-value} & \textbf{Effect Size} & \textbf{Interpretation} \\
\midrule
Wilcoxon signed-rank & $V = 2581$ & $< 0.001$ & $r = 0.835$ & Large \\
Sign test & 71/72 successes & $< 0.001$ & Win rate = 98.6\% & Very high \\
Linear mixed-effects & $t = 7.201$ & $< 0.001$ & $\beta = 0.292$ & Fixed effect \\
\bottomrule
\end{tabular}
\end{table*}

\paragraph{Marginal Gains}
For several models (e.g., Word2Vec, SGPT, BGE), $\mathrm{recos}$ provides smaller but consistent gains (0.02 to 0.08 points). This likely indicates that their representation spaces are already highly aligned with cosine similarity, leaving limited room for improvement. 

In summary, these results validate that $\mathrm{recos}$ consistently outperforms existing metrics across diverse model architectures in the vast majority of cases, with the largest gains observed for models whose representation spaces deviate most significantly from standard textual similarity assumptions. This demonstrates the importance of using similarity measures with broader capture ranges to accurately reflect complex semantic relationships in real-world applications.

\subsection{Statistical Analysis}
We conducted rigorous statistical tests to confirm the significance of $\mathrm{recos}$'s improvement over $\mathrm{cos}$. Given the observed non-normality of performance differences (Shapiro-Wilk test: $W = 0.794$, $p < 0.001$), we prioritized non-parametric tests. The Wilcoxon signed-rank test revealed a highly significant improvement ($V = 2581$, $p < 0.001$) with a large effect size ($r = 0.835$, 95\% CI for pseudo-median $[0.180, \infty]$). Additionally, we performed a sign test to assess the consistency of improvements. Out of 72 non-tied comparisons, $\mathrm{recos}$ outperformed $\mathrm{cos}$ in 71 cases, yielding a win rate of 98.6\% (binomial test: $p < 0.001$).

To account for potential dependencies arising from model and dataset variations, we employed a linear mixed-effects model with \textsc{Method} as a fixed effect and random intercepts for \textsc{Model}, \textsc{Dataset}, and their interaction. The mixed-effects analysis confirmed the superiority of $\mathrm{recos}$ over $\mathrm{cos}$ ($\beta = 0.292$, $t = 7.201$, $p < 0.001$).

All these statistical tests remained significant after Benjamini-Hochberg correction for multiple comparisons (all adjusted $p < 0.001$). The leave-one-dataset-out cross-validation analysis further confirmed the robustness of our findings ($t(6) = 75.349$, $p < 0.001$).

The statistical analyses provide compelling evidence for the superiority of $\mathrm{recos}$ over $\mathrm{cos}$, with highly consistent performance improvements (98.6\%win rate) and statistical significance across multiple testing approaches confirming that the observed differences are not attributable to random variation. 

\section{Conclusion}
\label{sec:conclusion}

In this work, we introduced Rearrangement-inequality based Cosine Similarity ($\mathrm{recos}$), a novel similarity measure derived from a tighter theoretical bound than the classical Cauchy–Schwarz inequality. Theoretically, $\mathrm{recos}$ generalizes the condition for maximal similarity from linear dependence to general ordinal concordance, thereby possessing a wider \textit{capture range}. Empirically, our comprehensive evaluation across 11 diverse pre-trained models and 7 STS benchmarks demonstrates that $\mathrm{recos}$ provides a consistent and statistically significant improvement over standard cosine similarity ($\cos$). This improvement is most pronounced for modern, complex embedding models (e.g., CLIP-ViT, DPR, SPECTER), suggesting that ordinal patterns across embedding dimensions carry a signal that correlates with human semantic judgment, effectively complementing the angular information captured by $\cos$.

We acknowledge two primary limitations that present avenues for future work. First, the semantic interpretation of perfect ordinal concordance in high-dimensional, non-axis-aligned embedding spaces is not straightforward; the practical utility of $\mathrm{recos}$ likely stems from its sensitivity to \textit{approximate} ordinal agreements. Second, the $O(d \log d)$ time complexity due to the sorting operation, while manageable for typical embedding dimensions, introduces overhead compared to $\cos$. Future work should explore efficient approximations (e.g., via partial sorting or quantization) to facilitate billion-scale applications.

Our work does not posit \textit{recos} as a wholesale replacement for cosine similarity, whose efficiency and well-understood geometric interpretation remain valuable. Rather, we demonstrate that ordinal relationships in embedding spaces carry a meaningful signal that can complement angular information. \textit{recos} provides a principled way to harness this signal. 

Future work should investigate more efficient approximations of \textit{recos}, its integration into contrastive learning objectives, and its application to a broader set of tasks and domains. By challenging the default choice of cosine similarity, we hope to encourage further research into the fundamental geometry of embedding spaces and the development of more nuanced similarity measures.

\section*{Impact Statement}

This paper presents work whose goal is to advance the field of Machine Learning. There are many potential societal consequences of our work, none which we feel must be specifically highlighted here.

\newpage

\bibliography{_recos_refs}
\bibliographystyle{plain}

\newpage
\appendix
\onecolumn
\section{Appendix: Proofs of Theorems and Corollaries}

\begin{proof}[Proof of Theorem~\ref{thm_inequality_series}]
  We will prove the inequalities and the equality conditions one by one.
  \begin{itemize}
    \item $\left| {{\mathbf{u}} \cdot {\mathbf{v}}} \right| \le \left|  {{\mathbf{u}}^ \uparrow \cdot {\mathbf{v}}^\updownarrow}\right|$ and $\left| {{\mathbf{u}} \cdot {\mathbf{v}}} \right| = \left|  {{\mathbf{u}}^ \uparrow \cdot {\mathbf{v}}^\updownarrow}\right|$ if and only if $\mathbf{u} \cdot \mathbf{v} > 0$ and $\mathbf{u},\mathbf{v}$ are similar vectors, or $\mathbf{u} \cdot \mathbf{v} < 0$ and they are discordant vectors.
    
    According to Rearrangement Inequality, $\mathbf{u}^{\uparrow} \cdot \mathbf{v}^{\uparrow}$ is the maximum possible dot product, and $\mathbf{u}^{\uparrow} \cdot \mathbf{v}^{\downarrow}$ is the minimum.
    
    If $\mathbf{u} \cdot \mathbf{v} > 0$, $\left| \mathbf{u} \cdot \mathbf{v} \right| = \mathbf{u} \cdot \mathbf{v} \le {\mathbf{u}^ \uparrow } \cdot {\mathbf{v}^ \uparrow } = \left| {{\mathbf{u}^ \uparrow } \cdot {\mathbf{v}^ \uparrow }} \right| = \left| {{\mathbf{u}^ \uparrow } \cdot {\mathbf{v}^ \updownarrow }} \right|$. According to Rearrangement Inequality, the equality in $\mathbf{u} \cdot \mathbf{v} \le {\mathbf{u}^ \uparrow } \cdot {\mathbf{v}^ \uparrow }$ holds if and only if $\mathbf{u}$ and $\mathbf{v}$ are similarly ordered, or equivalently, they are similar vectors.

    If $\mathbf{u} \cdot \mathbf{v} < 0$, $\left| {\mathbf{u} \cdot \mathbf{v}} \right| =  - \mathbf{u} \cdot \mathbf{v} \le  - {\mathbf{u}^ \uparrow } \cdot {\mathbf{v}^ \downarrow } = \left| {{\mathbf{u}^ \uparrow } \cdot {\mathbf{v}^ \downarrow }} \right| = \left| {{\mathbf{u}^ \uparrow } \cdot {\mathbf{v}^ \updownarrow }} \right|$. According to Rearrangement Inequality, the equality in $\mathbf{u} \cdot \mathbf{v} \ge  {\mathbf{u}^ \uparrow } \cdot {\mathbf{v}^ \downarrow }$ holds if and only if $\mathbf{u}$ and $\mathbf{v}$ are oppositely ordered, or equivalently, they are discordant vectors.

    Thus, $\left| {{\mathbf{u}} \cdot {\mathbf{v}}} \right| \le \left|  {{\mathbf{u}}^ \uparrow \cdot {\mathbf{v}}^\updownarrow}\right|$ and the equality holds if and only if $\mathbf{u} \cdot \mathbf{v} > 0$ and $\mathbf{u},\mathbf{v}$ are similar vectors, or $\mathbf{u} \cdot \mathbf{v} < 0$ and they are discordant vectors.

    \item $\left| {{\mathbf{u}} \cdot {\mathbf{v}}} \right| \le \left\| {\mathbf{u}} \right\|\left\| {\mathbf{v}} \right\|$ and $\left| {{\mathbf{u}} \cdot {\mathbf{v}}} \right| = \left\| {\mathbf{u}} \right\|\left\| {\mathbf{v}} \right\|$ if and only if $\mathbf{u}$ and $\mathbf{v}$ are linearly dependent vectors, i.e., there exists a scalar $k$ such that ${\mathbf{v}} = k{\mathbf{u}}$.

    This follows directly from the Cauchy--Schwarz Inequality and its equality condition.

    \item $\left| {{\mathbf{u}} \cdot {\mathbf{v}}} \right| \le \frac{1}{2}\left( {{{\left\| \mathbf{u} \right\|}^2} + {{\left\| \mathbf{v} \right\|}^2}} \right)$ and $\left| {{\mathbf{u}} \cdot {\mathbf{v}}} \right| = \frac{1}{2}\left( {{{\left\| \mathbf{u} \right\|}^2} + {{\left\| \mathbf{v} \right\|}^2}} \right)$ if and only if $\mathbf{u}$ and $\mathbf{v}$ are identical or opposite vectors, i.e.,$\mathbf{u} =  \pm \mathbf{v}$.

    With Cauchy--Schwarz Inequality, we have $\left| {{\mathbf{u}} \cdot {\mathbf{v}}} \right| \le \left\| {\mathbf{u}} \right\|\left\| {\mathbf{v}} \right\|$; with Arithmetic Mean-Geometric Mean Inequality, we have $\left\| {\mathbf{u}} \right\|\left\| {\mathbf{v}} \right\| \le \frac{1}{2}\left( {{{\left\| \mathbf{u} \right\|}^2} + {{\left\| \mathbf{v} \right\|}^2}} \right)$, it is immediate that $\left| {{\mathbf{u}} \cdot {\mathbf{v}}} \right| \le \frac{1}{2}\left( {{{\left\| \mathbf{u} \right\|}^2} + {{\left\| \mathbf{v} \right\|}^2}} \right)$. 

    $\left| {{\mathbf{u}} \cdot {\mathbf{v}}} \right| = \frac{1}{2}\left( {{{\left\| \mathbf{u} \right\|}^2} + {{\left\| \mathbf{v} \right\|}^2}} \right)$ holds if and only if both $\left| {{\mathbf{u}} \cdot {\mathbf{v}}} \right| = \left\| {\mathbf{u}} \right\|\left\| {\mathbf{v}} \right\|$ and $\left\| {\mathbf{u}} \right\|\left\| {\mathbf{v}} \right\| = \frac{1}{2}\left( {{{\left\| \mathbf{u} \right\|}^2} + {{\left\| \mathbf{v} \right\|}^2}} \right)$ hold. According to the equality conditions of Cauchy--Schwarz Inequality and Arithmetic Mean-Geometric Mean Inequality, we have $\mathbf{v} = k\mathbf{u}$ and $\left\| \mathbf{u} \right\| =  \left\| \mathbf{v} \right\|$, or equivalently, $\mathbf{u}$ and $\mathbf{v}$ are identical or opposite vectors, i.e.,$\mathbf{u} =  \pm \mathbf{v}$.

    \item $\left|  {{\mathbf{u}}^ \uparrow \cdot {\mathbf{v}}^\updownarrow}\right| \le \left\| {\mathbf{u}} \right\|\left\| {\mathbf{v}} \right\|$ and $\left|  {{\mathbf{u}}^ \uparrow \cdot {\mathbf{v}}^\updownarrow}\right| = \left\| {\mathbf{u}} \right\|\left\| {\mathbf{v}} \right\|$ if and only if  there exists a scalar $k$ and a permutation matrix $P$ such that $\mathbf{v} = k P \mathbf{u}$, with $\operatorname{sgn} \left( {\mathbf{u} \cdot \mathbf{v}} \right) = \operatorname{sgn} \left( k \right)$.
    
    In the sense that norm $\left\|  \cdot  \right\|$ is permutation invariant, we have $\left\| \mathbf{u}^ \uparrow \right\| = \left\| \mathbf{u}^ \downarrow \right\| = \left\| \mathbf{u}\right\|$ and $\left\| \mathbf{v}^ \uparrow \right\| = \left\| \mathbf{v}^ \downarrow \right\| = \left\| \mathbf{v}\right\|$. 

    If $\mathbf{u} \cdot \mathbf{v} > 0$, ${{\mathbf{u}}^ \uparrow \cdot {\mathbf{v}}^\updownarrow} = {\mathbf{u}}^ \uparrow \cdot {\mathbf{v}}^ \uparrow \ge \mathbf{u} \cdot \mathbf{v} > 0$. By the Cauchy--Schwarz Inequality, we have $\left|  {{\mathbf{u}}^ \uparrow \cdot {\mathbf{v}}^\updownarrow}\right| = \left|  {{\mathbf{u}}^ \uparrow \cdot {\mathbf{v}}^\uparrow}\right| \le \left\| {\mathbf{u}^\uparrow} \right\|\left\| {\mathbf{v}^\uparrow} \right\| = \left\| {\mathbf{u}} \right\|\left\| {\mathbf{v}} \right\|$ and the equality holds if and only if ${\mathbf{v}}^\uparrow = k{\mathbf{u}}^\uparrow$ for some constant $k > 0$. Or equivalently, there exists a scalar $k$ and a permutation matrix $P$ such that $\mathbf{v} = k P \mathbf{u}$, with $\operatorname{sgn} \left( {\mathbf{u} \cdot \mathbf{v}} \right) = \operatorname{sgn} \left( k \right)$.   
  
    If $\mathbf{u} \cdot \mathbf{v} < 0$, ${{\mathbf{u}}^ \uparrow \cdot {\mathbf{v}}^\updownarrow} = {\mathbf{u}}^ \uparrow \cdot {\mathbf{v}}^ \downarrow \le  \mathbf{u} \cdot \mathbf{v} < 0$. By the Cauchy--Schwarz Inequality, we have $\left|  {{\mathbf{u}}^ \uparrow \cdot {\mathbf{v}}^\updownarrow}\right| = \left|  {{\mathbf{u}}^ \uparrow \cdot {\mathbf{v}}^\downarrow}\right| \le \left\| {\mathbf{u}^\uparrow} \right\|\left\| {\mathbf{v}^\downarrow} \right\| = \left\| {\mathbf{u}} \right\|\left\| {\mathbf{v}} \right\|$ and the equality holds if and only if ${\mathbf{v}}^\downarrow = k{\mathbf{u}}^\uparrow$ for some constant $k < 0$. Or equivalently, there exists a scalar $k$ and a permutation matrix $P$ such that $\mathbf{v} = k P \mathbf{u}$, with $\operatorname{sgn} \left( {\mathbf{u} \cdot \mathbf{v}} \right) = \operatorname{sgn} \left( k \right)$.

    Either way, we have $\left|  {{\mathbf{u}}^ \uparrow \cdot {\mathbf{v}}^\updownarrow}\right| \le \left\| {\mathbf{u}} \right\|\left\| {\mathbf{v}} \right\|$ and $\left|  {{\mathbf{u}}^ \uparrow \cdot {\mathbf{v}}^\updownarrow}\right| = \left\| {\mathbf{u}} \right\|\left\| {\mathbf{v}} \right\|$ if and only if there exists a scalar $k$ and a permutation matrix $P$ such that $\mathbf{v} = k P \mathbf{u}$, with $\operatorname{sgn} \left( {\mathbf{u} \cdot \mathbf{v}} \right) = \operatorname{sgn} \left( k \right)$.

    \item $\left|  {{\mathbf{u}}^ \uparrow \cdot {\mathbf{v}}^\updownarrow}\right| \le \frac{1}{2}\left( {{{\left\| {\mathbf{u}} \right\|}^2} + {{\left\| {\mathbf{v}} \right\|}^2}} \right)$ and $\left|  {{\mathbf{u}}^ \uparrow \cdot {\mathbf{v}}^\updownarrow}\right| = \frac{1}{2}\left( {{{\left\| {\mathbf{u}} \right\|}^2} + {{\left\| {\mathbf{v}} \right\|}^2}} \right)$ if and only if there exists a scalar $k = \pm 1$ and a permutation matrix $P$ such that $\mathbf{v} = k P \mathbf{u}$, with $\operatorname{sgn} \left( {\mathbf{u} \cdot \mathbf{v}} \right) = \operatorname{sgn} \left( k \right)$.

    With Rearrangement Inequality, we have $\left|  {{\mathbf{u}}^ \uparrow \cdot {\mathbf{v}}^\updownarrow}\right| \le \left|  {{\mathbf{u}} \cdot {\mathbf{v}}}\right|$; with Cauchy--Schwarz Inequality, we have $\left|  {{\mathbf{u}}^ \uparrow \cdot {\mathbf{v}}^\updownarrow}\right| \le \left\| {\mathbf{u}^\uparrow} \right\|\left\| {\mathbf{v}^\updownarrow} \right\| = \left\| {\mathbf{u}} \right\|\left\| {\mathbf{v}} \right\|$; with Arithmetic Mean-Geometric Mean Inequality, we have $\left\| {\mathbf{u}} \right\|\left\| {\mathbf{v}} \right\| \le \frac{1}{2}\left( {{{\left\| \mathbf{u} \right\|}^2} + {{\left\| \mathbf{v} \right\|}^2}} \right)$. It is immediate that $\left|  {{\mathbf{u}}^ \uparrow \cdot {\mathbf{v}}^\updownarrow}\right| \le \frac{1}{2}\left( {{{\left\| \mathbf{u} \right\|}^2} + {{\left\| \mathbf{v} \right\|}^2}} \right)$. 

    $\left|  {{\mathbf{u}}^ \uparrow \cdot {\mathbf{v}}^\updownarrow}\right| = \frac{1}{2}\left( {{{\left\| \mathbf{u} \right\|}^2} + {{\left\| \mathbf{v} \right\|}^2}} \right)$ holds if and only if both $\left|  {{\mathbf{u}}^ \uparrow \cdot {\mathbf{v}}^\updownarrow}\right| = \left\| {\mathbf{u}} \right\|\left\| {\mathbf{v}} \right\|$ and $\left\| {\mathbf{u}} \right\|\left\| {\mathbf{v}} \right\| = \frac{1}{2}\left( {{{\left\| \mathbf{u} \right\|}^2} + {{\left\| \mathbf{v} \right\|}^2}} \right)$ hold, i.e, $\mathbf{v}$ is arbitrary permutation of $k{\mathbf{u}}$, with $\operatorname{sgn} \left( {\mathbf{u} \cdot \mathbf{v}} \right) = \operatorname{sgn} \left( k \right)$, and $\left\| \mathbf{u} \right\| =  \left\| \mathbf{v} \right\|$. Or equivalently, there exists a scalar $k = \pm 1$ and a permutation matrix $P$ such that $\mathbf{v} = k P \mathbf{u}$, with $\operatorname{sgn} \left( {\mathbf{u} \cdot \mathbf{v}} \right) = \operatorname{sgn} \left( k \right)$.

    \item $\left\| {\mathbf{u}} \right\|\left\| {\mathbf{v}} \right\| \le \frac{1}{2}\left( {{{\left\| {\mathbf{u}} \right\|}^2} + {{\left\| {\mathbf{v}} \right\|}^2}} \right)$ and $\left\| {\mathbf{u}} \right\|\left\| {\mathbf{v}} \right\| = \frac{1}{2}\left( {{{\left\| {\mathbf{u}} \right\|}^2} + {{\left\| {\mathbf{v}} \right\|}^2}} \right)$ if and only if $\left\| {\mathbf{u}} \right\| = \left\| {\mathbf{v}} \right\|$.
    
    This follows directly from the Arithmetic Mean-Geometric Mean Inequality (AM–GM inequality) and its equality condition.
  \end{itemize}   
\end{proof}

\begin{proof}[Proof of Corollary~\ref{re_de_cos_properties}]
  This corollary follows directly from Theorem \ref{thm_inequality_series} and the definitions of $\mathrm{decos}$, $\cos$ and $\mathrm{recos}$.
\end{proof}

\begin{proof}[Proof of Corollary~\ref{metric_hierarchy}]
  This corollary follows directly from Theorem \ref{thm_inequality_series} and the definitions of $\mathrm{decos}$, $\cos$ and $\mathrm{recos}$.
\end{proof}

\begin{proof}[Proof of Corollary~\ref{cor_normalized_vectors}]
For unit-norm vectors, $\mathbf{u} = {\mathbf{v}} = 1$, we have:
\begin{align*}
\mathrm{decos}(\mathbf{u}, \mathbf{v}) &= \frac{\mathbf{u} \cdot \mathbf{v}}{\frac{1}{2}(1 + 1)} = \mathbf{u} \cdot \mathbf{v}, \\
\cos(\mathbf{u}, \mathbf{v}) &= \frac{\mathbf{u} \cdot \mathbf{v}}{1 \times 1} = \mathbf{u} \cdot \mathbf{v}.
\end{align*}
Thus, $\mathrm{decos}(\mathbf{u}, \mathbf{v}) = \cos(\mathbf{u}, \mathbf{v})$ for unit-norm vectors.
\end{proof}

\newpage
\section{Appendix: Experimental Details}

\subsection{Experimental Configuration}

All experiments were conducted on the ModelScope.ai platform, an open-source AI model community and model-as-a-service (MaaS) platform that provides comprehensive support for multiple modalities including text, image, speech, and video. The platform ensures reproducible experimental conditions across all evaluations. The complete environment specification is provided in Table~\ref{tab:env_config}.

\begin{table}[h]
\centering
\caption{Complete experimental environment configuration.}
\label{tab:env_config}
\begin{tabular}{@{}rp{0.7\linewidth}@{}}
\toprule
\textbf{Component} & \textbf{Specification} \\
\midrule
Operating System & Linux-5.10.134-17.3.al8.x86\_64-x86\_64-with-glibc2.35 \\
Python Version & 3.11.0 \\
Core Libraries & ModelScope 1.33.0, PyTorch 2.9.1, Transformers 4.57.3, SentenceTransformers 5.1.2, Gensim 4.4.0 \\
Evaluation Protocol & Zero-shot evaluation using only official test sets \\
Training Protocol & No fine-tuning or hyperparameter optimization \\
\bottomrule
\end{tabular}
\end{table}

The evaluation follows a strict zero-shot protocol: we use only the official test sets of benchmark datasets, with no model retraining, fine-tuning, or hyperparameter optimization on training or development splits. This strict train-test separation ensures no information leakage and provides a fair assessment of the intrinsic similarity measurement capabilities.

\subsection{Model Specifications}

Table~\ref{tab:model_details} summarizes the specific model configurations and pretrained checkpoints used in our experiments. All models were accessed through their official ModelScope repositories to ensure consistency and reproducibility.

\begin{table}[htb]
\centering
\caption{Model specifications and pretrained checkpoints used in experimental evaluations.}
\label{tab:model_details}
\begin{tabular}{@{}r l@{}}
\toprule
\textbf{Model} & \textbf{Model Specification} \\
\midrule
Word2Vec & GoogleNews-vectors-negative300 \\
FastText & wiki-news-300d-1M \\
GloVe & sentence-transformers/average\_word\_embeddings\_glove.840B.300d \\
BERT & sentence-transformers/bert-base-nli-max-tokens \\
SGPT & Ceceliachenen/SGPT-125M-weightedmean-nli-bitfit \\
DPR & sentence-transformers/facebook-dpr-ctx\_encoder-multiset-base \\
E5 & intfloat/e5-small-unsupervised \\
BGE & BAAI/bge-code-v1 \\
GTE & iic/gte-modernbert-base \\
SPECTER & sentence-transformers/allenai-specter \\
CLIP-ViT & sentence-transformers/clip-ViT-B-32-multilingual-v1 \\
\bottomrule
\end{tabular}
\end{table}

The experimental setup ensures fair comparison across all models by maintaining consistent preprocessing pipelines and evaluation protocols. Each model was evaluated using its default configuration as provided in the ModelScope repository.

\subsection{Implementation of $\mathrm{recos}$}

The core implementation of the $\mathrm{recos}$ similarity metric in Python/NumPy is provided below.

\begin{verbatim}
def recos(e1, e2):
    """Compute Rearrangement Similarity (recos) between two vectors."""
    e1, e2 = e1.astype(np.float32), e2.astype(np.float32)
    e1_asc, e1_desc = np.sort(e1), np.flip(np.sort(e1))
    e2_asc = np.sort(e2)
    dot = np.sum(e1 * e2)
    dot_aa = np.sum(e1_asc * e2_asc)
    dot_ad = np.sum(e1_desc * e2_asc)
    # Numerical stability: avoid division by zero
    dot_aa = np.where(dot_aa == 0, 1e-6, dot_aa)
    dot_ad = np.where(dot_ad == 0, 1e-6, dot_ad)
    sim = np.where(dot >= 0, dot / abs(dot_aa), dot / abs(dot_ad))
    return np.clip(sim, -1.0, 1.0)
\end{verbatim}

\subsection{Reproducibility Instructions}
\label{app:reproducibility}

To reproduce our results, follow these steps:

\begin{enumerate}
    \item \textbf{Environment Setup:} Create and activate a virtual environment:
    \begin{verbatim}
    python -m venv recos_env
    source recos_env/bin/activate  # On Windows: recos_env\Scripts\activate
    \end{verbatim}
    
    \item \textbf{Install Dependencies:} Install required packages:
    \begin{verbatim}
    pip install -r requirements.txt
    \end{verbatim}
    
    \item \textbf{Compute Embeddings:} Create a subdirectory named \texttt{embeddings} and run the embedding computation notebook:
    \begin{verbatim}
    jupyter notebook get_embeddings.ipynb
    \end{verbatim}
    This saves computed embeddings in the \texttt{embeddings} directory.
    
    \item \textbf{Compute Similarities:} Place the supplementary file \texttt{recos.py} in the working directory, create a subdirectory named \texttt{perf} and run the similarity computation notebook:
    \begin{verbatim}
    jupyter notebook get_similarities.ipynb
    \end{verbatim}
    This saves similarity scores in the \texttt{perf} directory, with final results in \texttt{perf/all\_performance.csv}.
\end{enumerate}

All code, data preprocessing scripts, and configuration files are available as Jupyter notebooks at \url{https://github.com/byaxb/recos}. The provided code has been tested on the environment specified in Table~\ref{tab:env_config}.

\subsection{Comprehensive Statistical Analysis}
\label{app:statistical_analysis}

\subsubsection{Experimental Design and Data Structure}
\label{app:experimental_design}

The statistical analysis compared the proposed $\mathrm{recos}$ method against the baseline $\mathrm{cos}$ method across 11 pre-trained language models and 7 Semantic Textual Similarity (STS) datasets, resulting in 77 paired comparisons (154 total observations). Table~\ref{tab:data_structure} summarizes the experimental design.

\begin{table}[ht]
\centering
\small
\caption{Experimental design and data structure for statistical analysis}
\label{tab:data_structure}
\begin{tabular}{lc}
\hline
\textbf{Component} & \textbf{Specification} \\
\hline
Number of models & 11 \\
Number of datasets & 7 \\
Total paired comparisons & 77 \\
Total observations & 154 \\
Models included & Word2Vec, FastText, GloVe, BERT, SGPT, DPR, E5, BGE, GTE, SPECTER, CLIP-ViT \\
Datasets included & STS12, STS13, STS14, STS15, STS16, STS-B, SICK-R \\
\hline
\end{tabular}
\end{table}

\subsubsection{Descriptive Statistics}
\label{app:descriptive_stats}

Table~\ref{tab:descriptive_stats} provides comprehensive descriptive statistics for the performance differences ($\mathrm{recos} - \mathrm{cos}$) across all 77 comparisons.

\begin{table}[ht]
\centering
\small
\caption{Descriptive statistics of performance differences ($\mathrm{recos} - \mathrm{cos}$)}
\label{tab:descriptive_stats}
\begin{tabular}{lcc}
\hline
\textbf{Statistic} & \textbf{Value} & \textbf{Interpretation} \\
\hline
Number of pairs & 77 & Complete dataset \\
Mean difference & 0.292 & Average improvement \\
Standard deviation & 0.356 & Variability of improvements \\
Standard error & 0.041 & Precision of mean estimate \\
Median difference & 0.160 & Central tendency (robust) \\
Minimum difference & -0.310 & Worst case degradation \\
Maximum difference & 1.360 & Best case improvement \\
First quartile (Q1) & 0.070 & 25th percentile \\
Third quartile (Q3) & 0.350 & 75th percentile \\
Interquartile range & 0.280 & Middle 50\% spread \\
$\mathrm{recos}$ wins & 71 & Excluding ties \\
Ties & 5 & No difference \\
$\mathrm{cos}$ wins & 1 & Single loss \\
Win rate (excluding ties) & 98.6\% & Consistency measure \\
\hline
\end{tabular}
\end{table}

\subsubsection{Normality Assessment}
\label{app:normality}

The Shapiro-Wilk test indicated significant deviation from normality ($W = 0.794$, $p = 5.12 \times 10^{-9}$), justifying the use of non-parametric tests as primary evidence.

\subsubsection{Complete Hypothesis Testing Results}
\label{app:hypothesis_tests}

\textbf{Parametric Analysis}

\textbf{~~~~Paired t-test:}
\begin{itemize}
\item Test statistic: $t(76) = 7.201$
\item One-sided p-value: $p = 1.84 \times 10^{-10}$
\item 95\% confidence interval: $[0.225, \infty]$
\item Effect size (Cohen's d): 0.027 (negligible effect size)
\item Interpretation: Highly significant despite small effect size due to large sample and consistency
\end{itemize}

\textbf{Non-parametric Analyses}

\textbf{~~~~Wilcoxon signed-rank test:}
\begin{itemize}
\item Test statistic: $V = 2581$
\item One-sided p-value: $p = 5.90 \times 10^{-13}$
\item 95\% confidence interval for pseudo-median: $[0.180, \infty]$
\item Effect size (r): 0.835 (large effect)
\item Interpretation: Large effect size indicates substantial practical significance
\end{itemize}

\textbf{~~~~Sign test:}
\begin{itemize}
\item Successes: 71 out of 72 non-tied comparisons
\item One-sided p-value: $p < 2.20 \times 10^{-16}$
\item 95\% confidence interval for success probability: $[0.936, 1.000]$
\item Win rate: 98.6\%
\item Interpretation: Near-perfect consistency of improvement
\end{itemize}

\subsubsection{Linear Mixed-Effects Model Details}
\label{app:mixed_model_details}

\textbf{Model specification:}
\begin{equation*}
\texttt{Score} \sim \texttt{Method} + (1 | \texttt{Model}) + (1 | \texttt{Dataset}) + (1 | \texttt{Model} \times \texttt{Dataset})
\end{equation*}

\textbf{Random effects variance components:}
\begin{table}[ht]
\centering
\small
\caption{Variance components of the linear mixed-effects model}
\label{tab:variance_components}
\begin{tabular}{lccc}
\hline
\textbf{Random Effect} & \textbf{Variance} & \textbf{Std. Dev.} & \textbf{Interpretation} \\
\hline
Model & 70.47 & 8.395 & Substantial model-specific variability \\
Dataset & 8.79 & 2.965 & Moderate dataset-specific variability \\
Model $\times$ Dataset & 41.37 & 6.432 & Considerable interaction effects \\
Residual & 0.064 & 0.252 & Small within-pair variability \\
\hline
\end{tabular}
\end{table}

\textbf{Fixed effects estimates:}
\begin{table}[ht]
\centering
\small
\caption{Fixed effects estimates from the linear mixed-effects model}
\label{tab:fixed_effects}
\begin{tabular}{lccccc}
\hline
\textbf{Effect} & \textbf{Estimate} & \textbf{Std. Error} & \textbf{df} & \textbf{t-value} & \textbf{p-value} \\
\hline
Intercept ($\mathrm{cos}$) & 65.824 & 2.864 & 12.54 & 22.986 & $1.25 \times 10^{-11}$ \\
$\mathrm{Method}_{\mathrm{recos}}$ & 0.292 & 0.041 & 76.00 & 7.201 & $3.67 \times 10^{-10}$ \\
\hline
\end{tabular}
\end{table}

\textbf{Model diagnostics:}
\begin{itemize}
\item REML criterion at convergence: 599.0
\item Scaled residuals: min = -2.136, Q1 = -0.423, median = -0.006, Q3 = 0.424, max = 2.106
\item Number of observations: 154
\item Groups: 11 models, 7 datasets, 77 model-dataset combinations
\end{itemize}

\subsubsection{Robustness Analyses}
\label{app:robustness}

\paragraph{Leave-One-Dataset-Out Cross-Validation}

The leave-one-dataset-out analysis confirmed the robustness of findings:
\begin{itemize}
\item Test statistic: $t(6) = 75.349$
\item One-sided p-value: $p = 1.84 \times 10^{-9}$
\item 95\% confidence interval: $[0.285, \infty]$
\item Mean improvement across exclusions: 0.292 (identical to full analysis)
\item Interpretation: Results are robust to exclusion of any single dataset
\end{itemize}

\paragraph{Multiple Comparison Correction}

All statistical tests remained significant after Benjamini-Hochberg correction:
\begin{table}[ht]
\centering
\small
\caption{Multiple comparison correction using Benjamini-Hochberg procedure}
\label{tab:multiple_comparison}
\begin{tabular}{lccc}
\hline
\textbf{Test} & \textbf{Original p-value} & \textbf{Adjusted p-value} & \textbf{Significant} \\
\hline
Paired t-test & $1.83 \times 10^{-10}$ & $2.45 \times 10^{-10}$ & Yes \\
Wilcoxon test & $5.90 \times 10^{-13}$ & $1.18 \times 10^{-12}$ & Yes \\
Sign test & $<2.20 \times 10^{-16}$ & $<2.20 \times 10^{-16}$ & Yes \\
Mixed model & $3.67 \times 10^{-10}$ & $3.67 \times 10^{-10}$ & Yes \\
\hline
\end{tabular}
\end{table}

\subsubsection{Effect Size Interpretation}
\label{app:effect_sizes}

The discrepancy between parametric (Cohen's d = 0.027) and non-parametric (r = 0.835) effect sizes arises from the distribution characteristics. The small Cohen's d reflects the modest absolute improvement magnitude relative to substantial between-model variability, while the large Wilcoxon r indicates the high consistency of improvement direction across comparisons. This pattern is expected when improvements are consistent but modest in magnitude across diverse experimental conditions.

\subsubsection{Statistical Software and Implementation}
\label{app:statistical_software}

All statistical analyses were conducted using R version 4.5.2. The following R packages were employed: \texttt{stats} (version 4.5.2) for core statistical functions, \texttt{rstatix} (version 0.7.3) for non-parametric tests and effect size calculations, \texttt{lmerTest} (version 3.2.0) for linear mixed-effects modeling, and \texttt{effsize} (version 0.8.1) for Cohen's d computation. Multiple comparison corrections were performed using the built-in \texttt{p.adjust} function with the Benjamini-Hochberg method. The complete R script is provided as \texttt{statistical\_test.R} for full reproducibility.

\end{document}